\documentclass{article}

\usepackage{graphicx}
\usepackage{subfigure} 
\usepackage{floatrow}

\usepackage{amsmath,amsthm,amsfonts}

\usepackage{algorithm}
\usepackage[noend]{algorithmic}

\usepackage{url}

\newtheorem{theorem}{Theorem}
\newtheorem{proposition}{Proposition}

\theoremstyle{definition}
\newtheorem{definition}{Definition}

\title{Stochastic And-Or Grammars: A Unified Framework and Logic Perspective\thanks{This work was supported by the National Natural Science Foundation of China (61503248).}}
\author{Kewei Tu \\
School of Information Science and Technology\\
ShanghaiTech University, Shanghai, China\\
\texttt{tukw@shanghaitech.edu.cn}}

\date{}

\begin{document}

\maketitle
\begin{abstract}
Stochastic And-Or grammars (AOG) extend traditional stochastic grammars of language to model other types of data such as images and events.
In this paper we propose a representation framework of stochastic AOGs that is agnostic to the type of the data being modeled and thus unifies various domain-specific AOGs. Many existing grammar formalisms and probabilistic models in natural language processing, computer vision, and machine learning can be seen as special cases of this framework. We also propose a domain-independent inference algorithm of stochastic context-free AOGs and show its tractability under a reasonable assumption.
Furthermore, we provide two interpretations of stochastic context-free AOGs as a subset of probabilistic logic, which connects stochastic AOGs to the field of statistical relational learning and clarifies their relation with a few existing statistical relational models.
\end{abstract}

\section{Introduction}
\label{sec:intro}
Formal grammars are a popular class of knowledge representation that is traditionally confined to the modeling of natural and computer languages. However, several extensions of grammars have been proposed over time to model other types of data such as images \cite{fu1982syntactic,Zhu06as,Jin06ca} and events \cite{Ivanov00ro,Ryoo06ro,Pei11pv}. One prominent type of extension is stochastic And-Or grammars (AOG) \cite{Zhu06as}. A stochastic AOG simultaneously models compositions (i.e., a large pattern is the composition of several small patterns arranged according to a certain configuration) and reconfigurations (i.e., a pattern may have several alternative configurations), and in this way it can compactly represent a probabilistic distribution over a large number of patterns. Stochastic AOGs can be used to parse data samples into their compositional structures, which help solve multiple tasks (such as classification, annotation, and segmentation of the data samples) in a unified manner. 
In this paper we will focus on the context-free subclass of stochastic AOGs, which serves as the skeleton in building more advanced stochastic AOGs.

Several variants of stochastic AOGs and their inference algorithms have been proposed in the literature to model different types of data and solve different problems, such as image scene parsing \cite{Zhao11ip} and video event parsing \cite{Pei11pv}.
Our first contribution in this paper is that we provide \emph{a unified representation framework} of stochastic AOGs that is agnostic to the type of the data being modeled; in addition, based on this framework we propose \emph{a domain-independent inference algorithm} that is tractable under a reasonable assumption.
The benefits of a unified framework of stochastic AOGs include the following.
First, such a framework can help us generalize and improve existing ad hoc approaches for modeling, inference and learning with stochastic AOGs.
Second, it also facilitates applications of stochastic AOGs to novel data types and problems and enables the research of general-purpose inference and learning algorithms of stochastic AOGs.
Further, a formal definition of stochastic AOGs as abstract probabilistic models makes it easier to theoretically examine their relation with other models such as constraint-based grammar formalism \cite{shieber1992constraint} and sum-product networks \cite{Poon11}. In fact, we will show that many of these related models can be seen as special cases of stochastic AOGs.

Stochastic AOGs model compositional structures based on the relations between sub-patterns. Such probabilistic modeling of relational structures is traditionally studied in the field of statistical relational learning \cite{getoor2007introduction}.
Our second contribution is that we provide \emph{probabilistic logic interpretations} of the unified representation framework of stochastic AOGs and thus show that stochastic AOGs can be seen as a novel type of statistical relational models.
The logic interpretations help clarify the relation between stochastic AOGs and a few existing statistical relational models and probabilistic logics that share certain features with stochastic AOGs (e.g., tractable Markov logic \cite{domingos2012tractable} and stochastic logic programs \cite{muggleton1996stochastic}).
It may also facilitate the incorporation of ideas from statistical relational learning into the study of stochastic AOGs and at the same time contribute to the research of novel (tractable) statistical relational models.

\section{Stochastic And-Or Grammars}
\label{sec:aog}
An AOG is an extension of a constituency grammar used in natural language parsing \cite{Manning99book}. Similar to a constituency grammar, an AOG defines a set of valid hierarchical compositions of atomic entities. However, an AOG differs from a constituency grammar in that it allows atomic entities other than words and compositional relations other than string concatenation. A stochastic AOG models the uncertainty in the composition by defining a probabilistic distribution over the set of valid compositions.

Stochastic AOGs were first proposed to model images \cite{Zhu06as,Zhao11ip,Wang13hs,rothrock2013integrating}, in particular the spatial composition of objects and scenes from atomic visual words (e.g., Garbor bases). They were later extended to model events, in particular the temporal and causal composition of events from atomic actions \cite{Pei11pv} and fluents \cite{Fire13uc}. More recently, these two types of AOGs were used jointly to model objects, scenes and events from the simultaneous input of video and text \cite{tu2014joint}.

In each of the previous work using stochastic AOGs, a different type of data is modeled with domain-specific and problem-specific definitions of atomic entities and compositions. Tu et al. \cite{tu2013unsupervised} provided a first attempt towards a more unified definition of stochastic AOGs that is agnostic to the type of the data being modeled. We refine and extend their work by introducing parameterized patterns and relations in the unified definition, which allows us to reduce a wide range of related models to AOGs (as will be discussed in section \ref{sec:aog:related}).
Based on the unified framework of stochastic AOGs, we also propose a domain-independent inference algorithm and study its tractability (section \ref{sec:aog:inf}).
Below we start with the definition of stochastic context-free AOGs, which are the most basic form of stochastic AOGs and are used as the skeleton in building more advanced stochastic AOGs. 

A \emph{stochastic context-free AOG} is defined as a 5-tuple $\langle \Sigma, N, S, \theta, R \rangle$: 
\begin{description}
\item[$\Sigma$] is a set of terminal nodes representing atomic patterns that are not decomposable; 
\item[$N$] is a set of nonterminal nodes representing high-level patterns, which is divided into two disjoint sets: And-nodes and Or-nodes; 
\item[$S \in N$] is a start symbol that represents a complete pattern; 
\item[$\theta$] is a function that maps an instance of a terminal or nonterminal node $x$ to a parameter $\theta_x$ (the parameter can take any form such as a vector or a complex data structure; denote the maximal parameter size by $m_\theta$);
\item[$R$] is a set of grammar rules, each of which takes the form of $x \rightarrow C$ representing the generation from a nonterminal node $x$ to a set $C$ of nonterminal or terminal nodes (we say that the rule is ``headed'' by node $x$ and the nodes in $C$ are the ``child nodes'' of $x$).
\end{description}
The set of rules $R$ is further divided into two disjoint sets: And-rules and Or-rules.
\begin{itemize}
\item An And-rule, parameterized by a triple $\langle r, t, f \rangle$, represents the decomposition of a pattern into a configuration of non-overlapping sub-patterns. The And-rule specifies a production $r: A \rightarrow \{x_1, x_2, \ldots, x_n\}$ for some $n \geq 2$, where $A$ is an And-node and $x_1, x_2, \ldots, x_n$ are a set of terminal or nonterminal nodes representing the sub-patterns.
A relation between the parameters of the child nodes, $t(\theta_{x_1}, \theta_{x_2}, \ldots, \theta_{x_n})$, specifies valid configurations of the sub-patterns. This so-called \emph{parameter relation} is typically factorized to the conjunction of a set of binary relations. 
A \emph{parameter function} $f$ is also associated with the And-rule specifying how the parameter of the And-node $A$ is related to the parameters of the child nodes: $\theta_A = f(\theta_{x_1}, \theta_{x_2}, \ldots, \theta_{x_n})$. 
We require that both the parameter relation and the parameter function take time polynomial in $n$ and $m_\theta$ to compute.
There is exactly one And-rule that is headed by each And-node.
\item An Or-rule, parameterized by an ordered pair $\langle r, p \rangle$, represents an alternative configuration of a pattern. The Or-rule specifies a production $r: O \rightarrow x$, where $O$ is an Or-node and $x$ is either a terminal or a nonterminal node representing a possible configuration. A conditional probability $p$ is associated with the Or-rule specifying how likely the configuration represented by $x$ is selected given the Or-node $O$. The only constraint in the Or-rule is that the parameters of $O$ and $x$ must be the same: $\theta_O = \theta_x$.
There typically exist multiple Or-rules headed by the same Or-node, and together they can be written as $O \rightarrow x_1 | x_2 | \ldots | x_n$.
\end{itemize}

Note that unlike in some previous work, in the definition above we assume deterministic And-rules for simplicity. In principle, any uncertainty in an And-rule can be equivalently represented by a set of Or-rules each invoking a different copy of the And-rule.

Fig.\ \ref{fig:ex1}(a) shows an example stochastic context-free AOG of line drawings. Each terminal or nonterminal node represents an image patch and its parameter is a 2D vector representing the position of the patch in the image. Each terminal node denotes a line segment of a specific orientation while each nonterminal node denotes a class of line drawing patterns. The start symbol $S$ denotes a class of line drawing images (e.g., images of animal faces). In each And-rule, the parameter relation specifies the relative positions between the sub-patterns and the parameter function specifies the relative positions between the composite pattern and the sub-patterns.

\begin{figure*}[t]\centering
	\subfigure[]{\includegraphics[scale=.4]{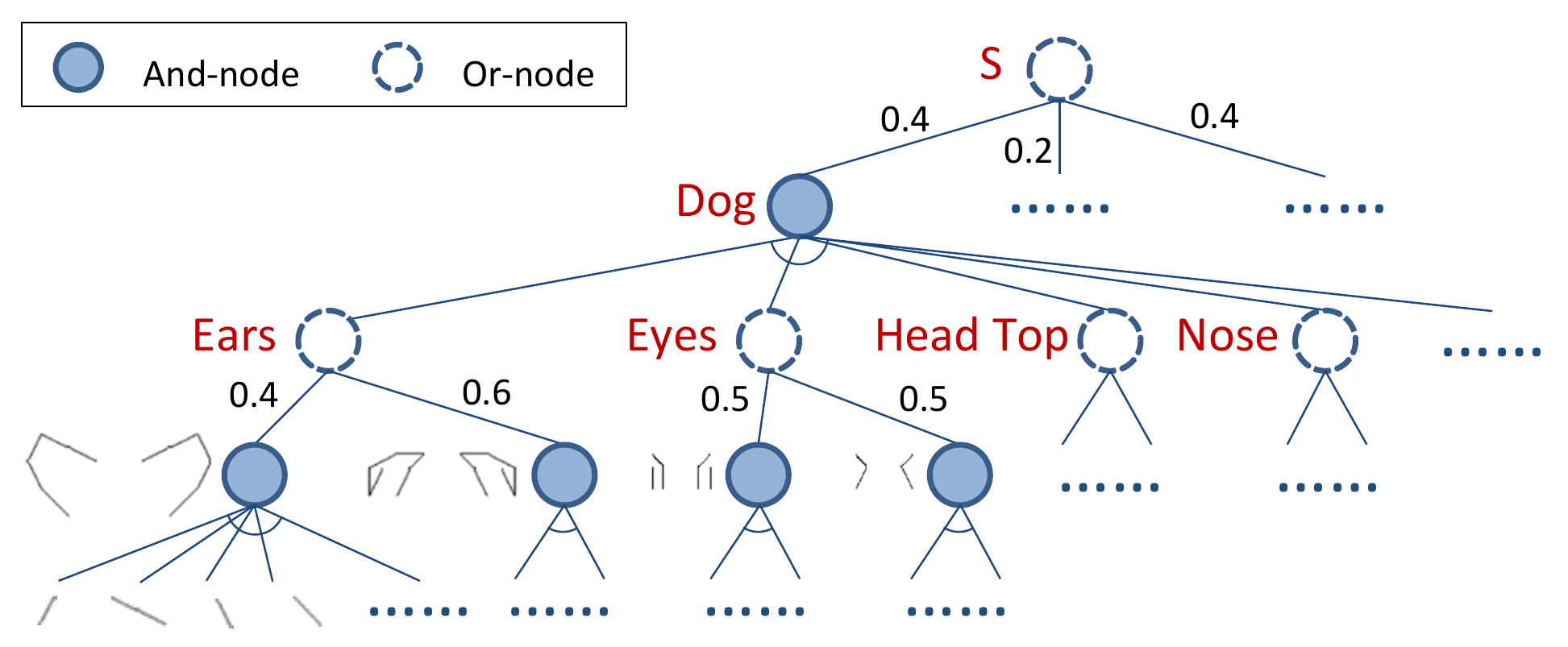}}
	\hspace{1ex}
	\subfigure[]{\includegraphics[scale=.4]{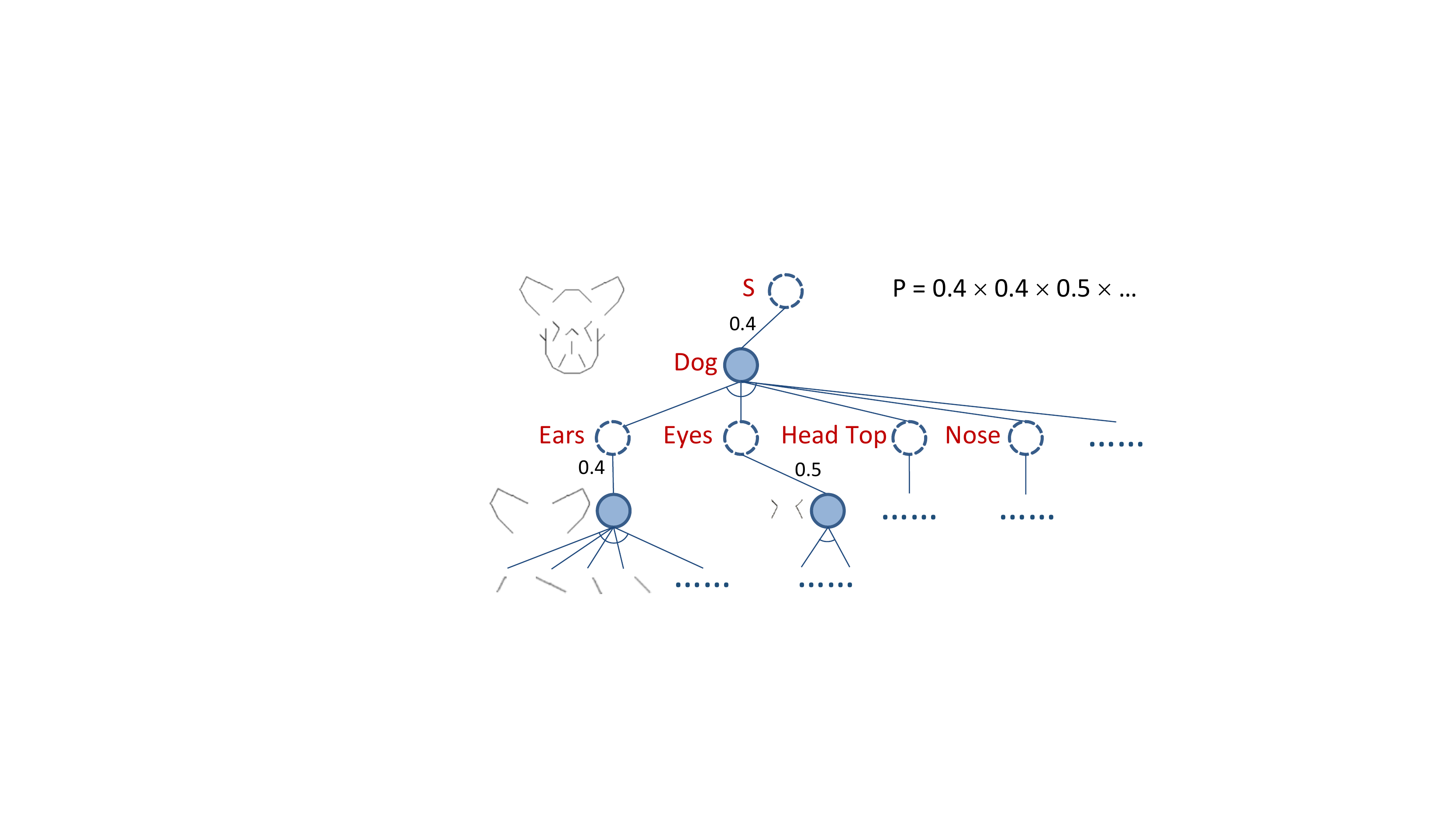}}
\caption{(a) A graphical representation of an example stochastic AOG of line drawings of animal faces. Each And-rule is represented by an And-node and all of its child nodes in the graph. The spatial relations within each And-rule are not shown for clarity. Each Or-rule is represented by an Or-node and one of its child nodes, with its probability shown on the corresponding edge.
(b) A line drawing image and its compositional structure generated from the example AOG. Again, the spatial relations between nodes are not shown for clarity. The probability of the compositional structure is partially computed at the top right.}
\label{fig:ex1}
\end{figure*}

With a stochastic context-free AOG, one can generate a compositional structure by starting from a data sample containing only the start symbol $S$ and recursively applying the grammar rules in $R$ to convert nonterminal nodes in the data sample until the data sample contains only terminal nodes. The resulting compositional structure is a tree in which the root node is $S$, each non-leaf node is a nonterminal node, and each leaf node is a terminal node; in addition, for each appearance of And-node $A$ in the tree, its set of child nodes in the tree conforms to the And-rule headed by $A$, and for each appearance of Or-node $O$ in the tree, it has exactly one child node in the tree which conforms to one of the Or-rules headed by $O$. 
The probability of the compositional structure is the product of the probabilities of all the Or-rules used in the generation process.
Fig.\ \ref{fig:ex1}(b) shows an image and its compositional structure generated from the example AOG in Fig.\ \ref{fig:ex1}(a).
Given a data sample consisting of only atomic patterns, one can also infer its compositional structure by parsing the data sample with the stochastic context-free AOG. We will discuss the parsing algorithm later.

Our framework is flexible in that it allows different types of patterns and relations within the same grammar. Consider for example a stochastic AOG modeling visually grounded events (e.g., videos of people using vending-machines). We would have two types of terminal or nonterminal nodes that model events and objects respectively. An event node represents a class of events or sub-events, whose parameter is the start/end time of an instance event. An object node represents a class of objects or sub-objects (possibly in a specific state or posture), whose parameter contains both the spatial information and the time interval information of an instance object. We specify temporal relations between event nodes to model the composition of an event from sub-events; we specify spatial relations between object nodes to model the composition of an object from its component sub-objects as well as the composition of an atomic event from its participant objects; we also specify temporal relations between related object nodes to enforce the alignment of their time intervals.

Note that different nonterminal nodes in an AOG may share child nodes. For example, in Fig.\ref{fig:ex1} each terminal node representing a line segment may actually be shared by multiple parent nonterminal nodes representing different line drawing patterns.
Furthermore, there could be recursive rules in an AOG, which means the direct or indirect production of a grammar rule may contain its left-hand side nonterminal. Recursive rules are useful in modeling languages and repetitive patterns.

In some previous work, stochastic AOGs more expressive than stochastic context-free AOGs are employed. A typical augmentation over context-free AOGs is that, while in a context-free AOG a parameter relation can only be specified within an And-rule, in more advanced AOGs parameter relations can be specified between any two nodes in the grammar. This can be very useful in certain scenarios. For example, in an image AOG of indoor scenes, relations can be added between all pairs of 2D faces to discourage overlap \cite{Zhao11ip}. However, such relations make inference much more difficult. 
Another constraint in context-free AOGs that is sometimes removed in more advanced AOGs is the non-overlapping requirement between sub-patterns in an And-rule. For example, in an image AOG it may be more convenient to decompose a 3D cube into 2D faces that share edges \cite{Zhao11ip}.
We will leave the formal definition and analysis of stochastic AOGs beyond context-freeness to future work.

\subsection{Related Models and Special Cases}\label{sec:aog:related}
Stochastic context-free AOGs subsume many existing models as special cases. Because of space limitation, here we informally describe these related models and their reduction to AOGs and leave the formal definitions and proofs in \ref{sm:sec:rmsc}.

Stochastic context-free grammars (SCFG) are clearly a special case of stochastic context-free AOGs. Any SCFG can be converted into an And-Or normal form that matches the structure of a stochastic AOG \cite{Tu08}. In a stochastic AOG representing a SCFG, each node represents a string and the parameter of a node is the start/end positions of the string in the complete sentence; the parameter relation and parameter function in an And-rule specify string concatenation, i.e., the substrings must be adjacent and the concatenation of all the substrings forms the composite string represented by the parent And-node.

There have been a variety of grammar formalisms developed in the natural language processing community that go beyond the concatenation relation of strings. For examples, in some formalisms the substrings are interwoven to form the composite string \cite{pollard1984generalized,johnson1985parsing}. More generally, in a grammar rule a linear regular string function can be used to combine lists of substrings into a list of composite strings, as in a linear context-free rewriting system (LCFRS) \cite{weir1988characterizing}.
All these grammar formalisms can be represented by context-free AOGs with each node representing a list of strings, the node parameter being a list of start/end positions, and in each And-rule the parameter relation and parameter function defining a linear regular string function. Since LCFRSs are known to generate the larger class of mildly context-sensitive languages, context-free AOGs when instantiated to model languages can be at least as expressive as mildly context-sensitive grammars.

Constraint-based grammar formalisms \cite{shieber1992constraint} are another class of natural language grammars, which associate so-called feature structures to nonterminals and use them to specify constraints in the grammar rules. Such constraints can help model natural language phenomena such as English subject-verb agreement and underlie grammatical theories such as head-driven phrase structure grammars \cite{Pollard1988information}. It is straightforward to show that constraint-based grammar formalisms are also special cases of context-free AOGs (with a slight generalization to allow unary And-rules), by establishing equivalence between feature structures and node parameters and between constraints and parameter relations/functions.


In computer vision and pattern recognition, stochastic AOGs have been applied to a variety of tasks as discussed in the previous section. In addition, several other popular models, such as the deformable part model \cite{felzenszwalb2008discriminatively} and the flexible mixture-of-parts model \cite{yang2011articulated}, can essentially be seen as special cases of stochastic context-free AOGs in which the node parameters encode spatial information of image patches and the parameter relations/functions encode spatial relations between the patches.

Sum-product networks (SPN) \cite{Poon11} are a new type of deep probabilistic models that extend the ideas of arithmetic circuits \cite{darwiche2003differential} and AND/OR search spaces \cite{dechter2007and} and can compactly represent many probabilistic distributions that traditional graphical models cannot tractably handle.
It can be shown that any decomposable SPN has an equivalent stochastic context-free AOG: Or-nodes and And-nodes of the AOG can be used to represent sum nodes and product nodes in the SPN respectively, all the node parameters are set to null, parameter relations always return true, and parameter functions always return null.
Because of this reduction, all the models that can reduce to decomposable SPNs can also be seen as special cases of stochastic context-free AOGs, such as thin junction trees \cite{bach2001thin}, mixtures of trees \cite{meila2001learning} and latent tree models \cite{choi2011learning}.

\subsection{Inference}
\label{sec:aog:inf}
The main inference problem associated with stochastic AOGs is parsing, i.e., given a data sample consisting of only terminal nodes, infer its most likely compositional structure (parse). A related inference problem is to compute the marginal probability of a data sample.
It can be shown that both problems are NP-hard (see \ref{sm:sec:np} for the proofs).
Nevertheless, here we propose an exact inference algorithm for stochastic context-free AOGs that is tractable under a reasonable assumption on the number of valid compositions in a data sample.
Our algorithm is based on bottom-up dynamic programming and can be seen as a generalization of several previous exact inference algorithms designed for special cases of stochastic AOGs (such as the CYK algorithm for text parsing).

Algorithm \ref{alg:inf} shows the inference algorithm that returns the probability of the most likely parse. After the algorithm terminates, the most likely parse can be constructed by recursively backtracking the selected Or-rules from the start symbol to the terminals. To compute the marginal probability of a data sample, we simply replace the max operation with sum in line \ref{alg:max2} of Algorithm \ref{alg:inf}.

In Algorithm \ref{alg:inf} we assume the input AOG is in a generalized version of Chomsky normal form, i.e., (1) each And-node has exactly two child nodes which must be Or-nodes, (2) the child nodes of Or-nodes must not be Or-nodes, and (3) the start symbol $S$ is an Or-node.
By extending previous studies \cite{lange2009cnf}, it can be shown that any context-free AOG can be converted into this form and both the time complexity of the conversion and the size of the new AOG is polynomial in the size of the original AOG. We give more details in \ref{sm:sec:cnf}.

\renewcommand{\algorithmicrequire}{\textbf{Input:}}
\renewcommand{\algorithmicensure}{\textbf{Output:}}
\begin{algorithm}[t]
\begin{algorithmic}[1]\small

	\REQUIRE a data sample $X$ consisting of a set of non-duplicate instances of terminal nodes, a stochastic context-free AOG $G$ in Chomsky normal form
	\ENSURE the probability $p^*$ of the most likely parse of $X$

	\STATE Create an empty map $M$ \quad \COMMENT{$M[i,O,\theta,T]$ stores the probability of a valid composition of size $i$ with root Or-node $O$, parameter $\theta$, and set $T$ of terminal instances.}
	\FORALL{$x \in X$} \label{alg:s1:s}
		\STATE $a \leftarrow$ the terminal node that $x$ is an instance of
		\STATE $\theta \leftarrow$ the parameter of $x$
		\FORALL{Or-rule $\langle O \rightarrow a,\ p \rangle$ in $G$}
			\STATE $M[1,O,\theta,\{x\}] \leftarrow p$
		\ENDFOR
	\ENDFOR \label{alg:s1:e}
	
	\FOR{$i=2$ \TO $|X|$} \label{alg:si:s}
		\FOR{$j=1$ \TO $i-1$}
			\FORALL{$\langle O_1, \theta_1, p_1 \rangle : M[j, O_1, \theta_1, T_1] = p_1$}
				\FORALL{$\langle O_2, \theta_2, p_2 \rangle : M[i-j, O_2, \theta_2, T_2] = p_2$}
					\FORALL{And-rule $\langle A \rightarrow O_1 O_2,\ t,\ f \rangle$ in $G$}
						\IF{$t(\theta_1, \theta_2) = True$ and $T_1 \bigcap T_2 = \emptyset$}
							\STATE $\phi \leftarrow f(\theta_1, \theta_2)$
							\STATE $T \leftarrow T_1 \bigcup T_2$
							\FORALL{Or-rule $\langle O \rightarrow A,\ p_O \rangle$ in $G$}
								\STATE $p \leftarrow p_O p_1 p_2$
								\IF{$M[i,O,\phi,T]$ is null}
									\STATE $M[i,O,\phi,T] \leftarrow p$
								\ELSE
									\STATE $M\![i,O,\phi,T]\!\!\leftarrow\!\max\{p,M\![i,O,\phi,T]\}$\label{alg:max2}
								\ENDIF
							\ENDFOR
						\ENDIF
					\ENDFOR
				\ENDFOR
			\ENDFOR
		\ENDFOR
	\ENDFOR \label{alg:si:e}
	
	\RETURN $\max_\theta M[|X|, S, \theta, X]$ \, \COMMENT{$S$ is the start symbol} \label{alg:ss}
\end{algorithmic}
\caption{Parsing with a stochastic context-free AOG}\label{alg:inf}
\end{algorithm}

The basic idea of Algorithm \ref{alg:inf} is to discover valid compositions of terminal instances of increasing sizes, where the size of a composition is defined as the number of terminal instances it contains. Size 1 compositions are simply the terminal instances (line \ref{alg:s1:s}--\ref{alg:s1:e}). To discover compositions of size $i>1$, the combination of any two compositions of sizes $j$ and $i-j\ (j<i)$ are considered (line \ref{alg:si:s}--\ref{alg:si:e}). A complete parse of the data sample is a composition of size $|X|$ with its root being the start symbol $S$ (line \ref{alg:ss}).

The time complexity of Algorithm \ref{alg:inf} is $O(|X|^2 c^2 |G|(|X|+|G|))$ where $c = \max_i |C_i|$ and $C_i$ is the set of valid compositions of size $i$ in the data sample $X$. 
In the worst case when all possible compositions of terminal instances from the data sample are valid, we have $c = \binom{|X|}{\left\lfloor|X|/2\right\rfloor}$ which is exponential in $|X|$. To make the algorithm tractable, we restrict the value of $c$ with the following assumption on the input data sample.
\theoremstyle{plain}
\newtheorem*{csa}{Composition Sparsity Assumption}
\begin{csa}
For any data sample $X$ and any positive integer $i \leq |X|$, the number of valid compositions of size $i$ in $X$ is polynomial in $|X|$.
\end{csa}
This assumption is reasonable in many scenarios. For text data, for a sentence of length $m$, a valid composition is a substring of the sentence and the number of substrings of size $i$ is $m-i+1$. For image data, if we restrict the compositions to be rectangular image patches (as in the hierarchical space tiling model \cite{Wang13hs}), then for an image of size $m = n \times n$ it is easy to show that the number of valid compositions of any specific size is no more than $n^3$.

\section{Logic Perspective of Stochastic AOGs}
\label{sec:logic}
In a stochastic AOG, And-rules model the relations between terminal and nonterminal instances and Or-rules model the uncertainty in the compositional structure. By combining these two types of rules, stochastic AOGs can be seen as probabilistic models of relational structures and are hence related to the field of statistical relational learning \cite{getoor2007introduction}. 
In this section, we manifest this connection by providing probabilistic logic interpretations of stochastic AOGs. By establishing this connection, we hope to facilitate the exchange of ideas and results between the two previously separated research areas.

\subsection{Interpretation as Probabilistic Logic}
\label{sec:logic:fol}
We first discuss an interpretation of stochastic context-free AOGs as a subset of first-order probabilistic logic with a possible-world semantics. The intuition is that we interpret terminal and nonterminal nodes of an AOG as unary relations, use binary relations to connect the instances of terminal and nonterminal nodes to form the parse tree, and use material implication to represent grammar rules.

We first describe the syntax of our logic interpretation of stochastic context-free AOGs.
There are two types of formulas in the logic: And-rules and Or-rules. Each And-rule takes the following form (for some $n \geq 2$).
\begin{multline*}
\forall x \, \exists y_1,y_2,\ldots,y_n, 
A(x) \rightarrow \bigwedge_{i=1}^n \left( B_i (y_i) \land R_i (x,y_i) \right) \\
\land R_\theta (\theta(x),\theta(y_1),\theta(y_2),\ldots,\theta(y_n))
\end{multline*}
The unary relation $A$ corresponds to the left-hand side And-node of an And-rule in the AOG; each unary relation $B_i$ corresponds to a child node of the And-rule. We require that for each unary relation $A$, there is at most one And-rule with $A(x)$ as the left-hand side. The binary relation $R_i$ is typically the \texttt{HasPart} relation between an object and one of its parts, but $R_i$ could also denote any other binary relation such as the \texttt{Agent} relation between an action and its initiator, or the \texttt{HasColor} relation between an object and its color. Note that these binary relations make explicit the nature of the composition represented by each And-rule of the AOG. $\theta$ is a function that maps an object to its parameter. $R_\theta$ is a relation that combines the parameter relation and parameter function in the And-rule of the AOG and is typically factorized to the conjunction of a set of binary relations. 

Each Or-rule takes the following form.
\[
\forall x, A(x) \rightarrow B(x) \; : p
\]
The unary relation $A$ corresponds to the left-hand side Or-node and $B$ to the child node of an Or-rule in the AOG;
$p$ is the conditional probability of $A(x) \rightarrow B(x)$ being true when the grounded left-hand side $A(x)$ is true. We require that for each true grounding of $A(x)$, among all the grounded Or-rules with $A(x)$ as the left-hand side, exactly one is true. This requirement can be represented by two additional sets of constraint rules. First, Or-rules with the same left-hand side are mutually exclusive, i.e., for any two Or-rules $\forall x, A(x) \rightarrow B_i(x)$ and $\forall x, A(x) \rightarrow B_j(x)$, we have $\forall x, A(x) \rightarrow B_i(x) \uparrow B_j(x)$ where $\uparrow$ is the Sheffer stroke. Second, given a true grounding of $A(x)$, the Or-rules with $A(x)$ as the left-hand side cannot be all false, i.e., $\forall x, A(x) \rightarrow \bigvee_i B_i(x)$ where $i$ ranges over all such Or-rules.
Further, to simplify inference and avoid potential inconsistency in the logic, we require that the right-hand side unary relation $B$ of an Or-rule cannot appear in the left-hand side of any Or-rule (i.e., the second requirement in the generalized Chomsky normal form of AOG described earlier).

We can divide the set of unary relations into two categories: those that appear in the left-hand side of rules (corresponding to the nonterminal nodes of the AOG) and those that do not (corresponding to the terminal nodes). The first category is further divided into two sub-categories depending on whether the unary relation appears in the left-hand side of And-rules or Or-rules (corresponding to the And-nodes and Or-nodes of the AOG respectively). We require these two sub-categories to be disjoint. There is also a unique unary relation $S$ that does not appear in the right-hand side of any rule, which corresponds to the start symbol of the AOG. 

Now we describe the semantics of the logic. The interpretation of all the logical and non-logical symbols follows that of first-order logic.
There are two types of objects in the universe of the logic: normal objects and parameters. There is a bijection between normal objects and parameters, and function $\theta$ maps a normal object to its corresponding parameter.
A possible world is represented by a pair $\langle X,L \rangle$ where $X$ is a set of objects and $L$ is a set of literals that are true. We require that there exists exactly one normal object $s \in X$ such that $S(s) \in L$. In order for all the deterministic formulas (i.e., all the And-rules and the two sets of constraint rules of all the Or-rules) to be satisfied, the possible world must contain a tree structure in which:
\begin{enumerate}
\item each node denotes an object in $X$ with the root node being $s$;
\item each edge denotes a binary relation defined in some And-rule;
\item for each leaf node $x$, there is exactly one terminal unary relation $T$ such that $T(x) \in L$;
\item for each non-leaf node $x$, there is exactly one And-node unary relation $A$ such that $A(x) \in L$, and for the child nodes $\{y_1,y_2,\ldots,y_n\}$ of $x$ in the tree, $\{B_i (y_i)\}_{i=1}^n \cup \{R_i (x,y_i)\}_{i=1}^n \cup \{R_\theta (\theta(x),\theta(y_1),\theta(y_2),\ldots,\theta(y_n))\} \subseteq L$ according to the And-rule associated with relation $A$;
\item for each node $x$, if for some Or-node unary relation $A$ we have $A(x) \in L$, then among all the Or-rules with $A$ as the left-hand side, there is exactly one Or-rule such that $B(x) \in L$ where $B$ is the right-hand side unary relation of the Or-rule, and for the rest of the Or-rules we have $\lnot B(x) \in L$.
\end{enumerate}
We enforce the following additional requirements to ensure that the possible world contains no more and no less than the tree structure:
\begin{enumerate}
\item No two nodes in the tree denote the same object.
\item $X$ and $L$ contain only the objects and relations specified above.
\end{enumerate}

The probability of a possible world $\langle X,L \rangle$ is defined as follows. Denote by $R^{Or}$ the set of Or-rules. For each Or-rule $r: \forall x, A(x) \rightarrow B(x)$, denote by $p_r$ the conditional probability associated with $r$ and define $g_r := \{x\in X | A(x)\in L\ \land\ B(x)\in L\}$. Then we have:
\[
	P(\langle X,L \rangle) = \prod_{r \in R^{Or}} {p_r}^{|g_r|}
\]

In this logic interpretation, parsing corresponds to the inference problem of identifying the most likely possible world in which the terminal relations and parameters of the leaf nodes of the tree structure match the atomic patterns in the input data sample. Computing the marginal probability of a data sample corresponds to computing the probability summation of the possible worlds that match the data sample. 


Our logic interpretation of stochastic context-free AOGs resembles tractable Markov logic (TML) \cite{domingos2012tractable,webb2013tractable} in many aspects, even though the two have very different motivations. 
Such similarity implies a deep connection between stochastic AOGs and TML and points to a potential research direction of investigating novel tractable statistical relational models by borrowing ideas from the stochastic grammar literature.
There are a few minor differences between stochastic AOGs and TML, e.g.,
TML does not distinguish between And-nodes and Or-nodes,
does not allow recursive rules, 
enforces that the right-hand side unary relation in each Or-rule is a sub-type of the left-hand side unary relation, and disallows a unary relation to appear in the right-hand side of more than one Or-rule. 

\subsection{Interpretation as a Stochastic Logic Program}
Stochastic logic programs (SLP) \cite{muggleton1996stochastic} are a type of statistical relational models that, like stochastic context-free AOGs, are a generalization of stochastic context-free grammars. They are essentially equivalent to two other representations, independent choice logic \cite{poole1993probabilistic} and PRISM \cite{sato2001parameter}.
Here we show how a stochastic context-free AOG can be represented by a pure normalized SLP \cite{cussens2001parameter}. Since several inference and learning algorithms have been developed for SLPs and PRISM, our reduction enables the application of these algorithms to stochastic AOGs.

In our SLP program, we have one SLP clause for each And-rule and each Or-rule in the AOG. The overall structure is similar to the probabilistic logic interpretation discussed in section \ref{sec:logic:fol}.
For each And-rule, the corresponding SLP clause takes the following form:
\begin{align*}
1.0: & \; a(X,P) \textrm{ :- }  b_1(X_1,P_1), b_2(X_2,P_2), \cdots, b_n(X_n,P_n), \\
	& \hspace{2.2em} append([X_1,\ldots,X_n],X), r_1(X,X_1), r_2(X,X_2), \\
	& \hspace{2.2em} \cdots, r_n(X,X_n), r_\theta(P,P_1,\ldots,P_n).
\end{align*}
The head $a(X,P)$ represents the left-hand side And-node of the And-rule, where $X$ represents the set of terminal instances generated from the And-node and $P$ is the parameters of the And-node. In the body of the clause, $b_i$ represents the $i$-th child node of the And-rule, $r_i$ represents the relation between the And-node and its $i$-th child node, $append(\ldots)$ states that the terminal instance set $X$ of the And-node is the union of the instance sets from all the child nodes, and $r_\theta$ represents a relation that combines the parameter relation and parameter function of the And-rule.
For relations $r_i$ and $r_\theta$, we need to have additional clauses to define them according to the type of data being modeled.

For each Or-rule in the AOG, if the right-hand side is a nonterminal, then we have:
\[
p: \; a(X,P) \textrm{ :- } b(X,P).
\]
where $p$ is the conditional probability associated with the Or-rule, $a$ and $b$ represent the left-hand and right-hand sides of the Or-rule respectively, whose arguments $X$ and $P$ have the same meaning as explained above.
If the right-hand side of the Or-rule is a terminal, then we have:
\[
p: \; a([t],[\ldots]).
\]
where $t$ is the right-hand side terminal node and the second argument represents the parameters of the terminal node. 

Finally, the goal of the program is
\[
\textrm{:-} \ s(X,P).
\]
which represents the start symbol of the AOG, whose arguments have the same meaning as explained above.


\section{Conclusion}
Stochastic And-Or grammars extend traditional stochastic grammars of language to model other types of data such as images and events.
We have provided a unified representation framework of stochastic AOGs that can be instantiated for different data types. We have shown that many existing grammar formalisms and probabilistic models in natural language processing, computer vision, and machine learning can all be seen as special cases of stochastic context-free AOGs. We have also proposed an inference algorithm for parsing data samples using stochastic context-free AOGs and shown that the algorithm is tractable under the composition sparsity assumption.
In the second part of the paper, we have provided interpretations of stochastic context-free AOGs as a subset of first-order probabilistic logic and stochastic logic programs. Our interpretations connect stochastic AOGs to the field of statistical relational learning and clarify their relation with a few existing statistical relational models.



\bibliography{aog}
\bibliographystyle{unsrt}

\newpage
\normalsize
\appendix
\gdef\thesection{Appendix \Alph{section}}

\section{Related Models and Special Cases}\label{sm:sec:rmsc}
\subsection{Stochastic Context-Free Grammars}
\begin{definition}
	A stochastic context-free grammar (SCFG) is a 4-tuple $\langle \Sigma,N,S,R \rangle$:
	\begin{itemize}
		\item $\Sigma$ is a set of terminal symbols
		\item $N$ is a set of nonterminal symbols
		\item $S$ is a special nonterminal called the start symbol
		\item $R$ is a set of production rules, each of the form $A \rightarrow \alpha\ [p]$ where $A \in N$, $\alpha \in (\Sigma \bigcup N)^*$, and $p$ is the conditional probability $P(\alpha|A)$.
	\end{itemize}
\end{definition}

Any SCFG can be converted into \emph{And-Or normal form} as described in \cite{Tu08}. The conversion results in a linear increase in the grammar size.
\begin{definition}
	An SCFG is in And-Or normal form iff. its nonterminal symbols are divided into two disjoint subsets: And-symbols and Or-symbols, such that:
	\begin{itemize}
		\item each And-symbol appears on the left-hand side of exactly one production rule, and the right-hand side of the rule contains a sequence of two or more terminal or nonterminal symbols;
		\item each Or-symbol appears on the left-hand side of one or more rules, each of which has a single terminal or nonterminal symbol on the right-hand side.
	\end{itemize}
\end{definition}

\begin{proposition}
	Any SCFG can be converted into And-Or normal form with linear increase in size.
\end{proposition}
\begin{proof}
	We construct a SCFG in And-Or normal form as follows.
	For each production rule $A \rightarrow \alpha\ [p]$ with two or more symbols in $\alpha$, create an And-symbol $B$ and replace the production rule with two new rules: $A\rightarrow B\ [p]$ and $B\rightarrow\alpha\ [1.0]$. Regard all the nonterminals in the original SCFG as Or-symbols.
\end{proof}

\begin{proposition}
	Any SCFG can be represented by a stochastic context-free AOG with linear increase in size.
\end{proposition}
\begin{proof}
	We first convert the SCFG into And-Or normal form. We then construct an equivalent stochastic context-free AOG $\langle \Sigma,N,S,\theta,R \rangle$:
	\begin{itemize}
		\item $\Sigma$ is the set of terminal symbols in the SCFG.
		\item $N$ is the set of nonterminal symbols in the SCFG, with a correspondence from And-symbols to And-nodes and from Or-symbols to Or-nodes.
		\item $S$ is the start symbol of the SCFG.
		\item $\theta$ maps a substring represented by a terminal or nonterminal symbol to its start/end positions in the complete sentence.
		\item $R$ is constructed from the set of production rules in the And-Or normal form SCFG; each rule headed by an And-symbol becomes an And-rule, with its parameter relation specifies that the substrings represented by the child nodes must be adjacent (by checking their start/end positions) and its parameter function outputs the start/end positions of the concatenated string represented by the parent And-node (i.e., the start position of the leftmost substring and the end position of the rightmost substring); each rule headed by an Or-symbol becomes an Or-rule with the same conditional probability.
	\end{itemize}
	It is easy to verify that the size of the stochastic context-free AOG is linear in the size of the original SCFG.
\end{proof}

\subsection{Linear Context-Free Rewriting Systems}
Linear context-free rewriting systems (LCFRS) \cite{weir1988characterizing} are a class of mildly context-sensitive grammars, which subsume as special cases a few other grammar formalisms \cite{pollard1984generalized,johnson1985parsing}.
\begin{definition}
	A linear context-free rewriting system is a 4-tuple $\langle \Sigma,N,S,R \rangle$:
	\begin{itemize}
		\item $\Sigma$ is a set of terminal symbols
		\item $N$ is a set of nonterminal symbols
		\item $S$ is a special nonterminal called the start symbol
		\item $R$ is a set of production rules, each of the form $p: A[g(\beta_1,\ldots,\beta_r)] \rightarrow B_1[\beta_1], \ldots, B_r[\beta_r]$ such that $p$ is the conditional probability of the rule given $A$, $A,B_1,\ldots,B_r \in N$, $\beta_i\in V^{\phi(B_i)}$ (for $i=1,\ldots,r$) where $\phi: N\rightarrow\mathbb{N}$ specifies the \emph{fan-out} of a nonterminal symbol and $V$ is a set of variables, and $g: V^{\phi(B_1)}\times\ldots\times V^{\phi(B_r)}\rightarrow((V\bigcup \Sigma)^+)^{\phi(A)}$ is a \emph{composition function} that is \emph{linear} and \emph{regular}, i.e., in the equation
		\[
		g(\beta_1,\ldots,\beta_r) = \langle t_1, \ldots, t_{\phi(A)} \rangle
		\]
		each variable in $V$ appears at most once on each side of the equation and the two sides of the equation contain exactly the same set of variables.
	\end{itemize}
\end{definition}

We can define And-Or normal form of LCFRS in a similar way as for SCFG.
\begin{definition}
	An LCFRS is in And-Or normal form iff. its nonterminal symbols are divided into two disjoint subsets: And-symbols and Or-symbols, such that:
	\begin{itemize}
		\item each And-symbol appears on the left-hand side of exactly one production rule, and the number of nonterminal symbols on right-hand side of the rule plus the number of terminals inserted by the composition function is larger than or equal to two;
		\item each Or-symbol appears on the left-hand side of one or more rules, in each of which the number of nonterminal symbols on right-hand side plus the number of terminals inserted by the composition function is one.
	\end{itemize}
\end{definition}

\begin{proposition}
	Any LCFRS can be converted into And-Or normal form with linear increase in size.
\end{proposition}
\begin{proof}
	The conversion can be done in the same way as for SCFG.
\end{proof}

\begin{proposition}
	Any LCFRS can be represented by a stochastic context-free AOG with linear increase in size.
\end{proposition}
\begin{proof}
	We first convert the LCFRS into And-Or normal form.
	We then construct an equivalent stochastic context-free AOG $\langle \Sigma,N,S,\theta,R \rangle$:
	\begin{itemize}
		\item $\Sigma$ is the set of terminal symbols in the LCFRS.
		\item $N$ is the set of nonterminal symbols in the LCFRS, with a correspondence from And-symbols to And-nodes and from Or-symbols to Or-nodes.
		\item $S$ is the start symbol of the LCFRS.
		\item $\theta$ maps a list of substrings represented by a terminal or nonterminal symbol to a list of start/end positions of these substrings in the complete sentence.
		\item $R$ is constructed from the set of production rules in the And-Or normal form LCFRS:
		\begin{itemize}
			\item Each rule headed by an And-symbol becomes an And-rule, whose right-hand side includes all the right-hand side nonterminal symbols of the original rule as well as all the terminal symbols added by the composition function. Note that each of the substrings represented by the And-symbol is formed by the composition function by concatenating terminals and/or substrings represented by the nonterminal symbols on the right-hand side of the rule. The parameter relation enforces that these component substrings are adjacent (by checking their start/end positions), and the parameter function outputs the start/end positions of the concatenated strings.
			\item Each rule headed by an Or-symbol becomes an Or-rule with the same conditional probability, whose right-hand side contains the single right-hand side nonterminal symbol of the original rule or the single terminal symbol from the composition function.
		\end{itemize}
	\end{itemize}
	It is easy to verify that the size of the stochastic context-free AOG is linear in the size of the original LCFRS.
\end{proof}

\subsection{Constraint-based Grammar Formalisms}
Constraint-based grammar formalisms \cite{shieber1992constraint} associate feature structures to nonterminals and use them to specify constraints in the grammar rules.
\begin{definition}
	A \emph{feature structure} is a set of attribute-value pairs. The value of an attribute is either an atomic symbol or another feature structure.
	A \emph{feature path} in a feature structure is a list of attributes that leads to a particular value.
\end{definition}
Below is an example feature structure, and $\langle$Agreement Number$\rangle$ is a feature path leading to the atomic symbol value \textit{singular}.
\[\left[\begin{array}{ll}
\textrm{Category} & \textit{NP} \\
\textrm{Agreement} & \left[
\begin{array}{ll}
\textrm{Number} & \textit{singular} \\
\textrm{Person} & \textit{third}
\end{array}
\right]
\end{array}\right]\]

\begin{definition}
	A constraint-based grammar formalism is a 4-tuple $\langle \Sigma,N,S,R \rangle$:
	\begin{itemize}
		\item $\Sigma$ is a set of terminal symbols
		\item $N$ is a set of nonterminal symbols
		\item $S$ is a special nonterminal called the start symbol
		\item $R$ is a set of production rules, each of the form $p: A \rightarrow \alpha\ \{C\}$ where $p$ is the conditional probability $P(\alpha|A)$, $A \in N$, $\alpha \in (\Sigma \bigcup N)^*$, and $C$ is a set of \emph{feature constraints};
		each nonterminal symbol in the rule is associated with a feature structure;
		each feature constraint takes the form of either ``$\langle X$ feature-path$\rangle =$ atomic-value'' or ``$\langle X$ feature-path$\rangle = \langle Y$ feature-path$\rangle$'', where $X,Y$ are nonterminal symbols in the rule.
	\end{itemize}
\end{definition}

\begin{proposition}
	Any constraint-based grammar formalism can be represented with linear increase in size by a generalization of stochastic context-free AOG that allows an And-rule to have only one symbol on the right-hand side.
\end{proposition}
\begin{proof}
	We construct an equivalent stochastic context-free AOG $\langle \Sigma,N,S,\theta,R \rangle$ in which we allow an And-rule to have only one symbol on the right-hand side:
	\begin{itemize}
		\item $\Sigma$ is the set of terminal symbols in the constraint-based grammar formalism.
		\item For $N$, all the nonterminal symbols of the constraint-based grammar formalism become Or-nodes, and for each production rule we create an And-node.
		\item $S$ is the start symbol of the constraint-based grammar formalism.
		\item $\theta$ maps a word represented by a terminal symbol to the start/end positions of the word in the complete sentence and maps a substring represented by a nonterminal symbol to a feature structure in addition to the start/end positions of the substring.
		\item $R$ is constructed as follows. For each rule $p: A \rightarrow \alpha\ \{C\}$ in the constraint-based grammar formalism, create one Or-rule $p: A \rightarrow B$ and one And-rule $B \rightarrow \alpha$ where $B$ is a new And-node. Suppose $C'$ is a copy of $C$ with all the appearance of $A$ changed to $B$. Then the parameter relation of the And-rule is the conjunction of the constraints in $C'$ that does not involve $B$ plus the constraint that the substrings represented by the child nodes must be adjacent (by checking their start/end positions); the parameter function outputs the start/end positions of the concatenated string as well as a new feature structure constructed according to the constraints in $C'$ that involve $B$.
	\end{itemize}
	It is easy to verify that the size of the stochastic context-free AOG is linear in the size of the original constraint-based grammar formalism.
\end{proof}

\subsection{Sum-Product Networks}
Sum-product networks (SPN) \cite{Poon11} are a new type of deep probabilistic models that can be more compact than traditional graphical models.
\begin{definition}
	A sum-product network over random variables $x_1, x_2, \ldots, x_d$ is a rooted directed acyclic graph. Each leaf node is an indicator $x_i$ or $\bar{x}_i$. Each non-leaf node is either a sum node or a product node. A sum node computes a weighted sum of its child nodes. A product node computes the product of its child nodes. The value of an SPN is the value of its root node. The scope of a node is the set of variables appearing in its descendant leaf nodes. For an SPN to correctly compute the probability of all evidence, the children of any sum node must have identical scopes and the children of any product node cannot contain conflicting descendant leaf nodes (i.e., $x_i$ in one child and $\bar{x}_i$ in another).
\end{definition}
\begin{definition}
	A decomposable SPN is an SPN in which the children of any product node have disjoint scopes.
\end{definition}
It has been shown that any SPN can be converted into a decomposable SPN with polynomial increase in size \cite{peharz2015theoretical}.

\begin{proposition}
	Any decomposable SPN can be represented by a stochastic context-free AOG with linear increase in size.
\end{proposition}
\begin{proof}
	We construct an equivalent stochastic context-free AOG $\langle \Sigma,N,S,\theta,R \rangle$:
	\begin{itemize}
		\item $\Sigma$ is the set of leaf nodes (indicators) in the SPN.
		\item $N$ is the set of non-leaf nodes in the SPN, with a correspondence from product nodes to And-nodes and from sum nodes to Or-nodes.
		\item $S$ is the root node of the SPN.
		\item $\theta$ maps any node instance to null (i.e., we set all the instance parameters to null).
		\item $R$ is constructed as follows: for each product node in the SPN, create an And-rule with the product node as the left-hand side and the set of child nodes as the right-hand side, let the parameter relation be always true, and let the parameter function always return null; for each child node of each sum node in the SPN, create an Or-rule with the sum node as the left-hand side, the child node as the right-hand side, and the normalized weight of the child node as the conditional probability.
	\end{itemize}
	As shown in \cite{peharz2015theoretical}, normalization of the child node weights of the sum nodes do not change the distribution modeled by the SPN. Therefore, for any assignment to the random variables, the marginal probability computed by the constructed stochastic context-free AOG and the probability computed by the original SPN are always equivalent. It is easy to verify that the size of the stochastic context-free AOG is linear in the size of the original SPN.
\end{proof}

Note that although SPNs are also general-purpose probabilistic models that can be used in modeling many types of data, stochastic AOGs go beyond SPNs in a few important aspects. Specifically, stochastic AOGs can simultaneously model data samples of different sizes, explicitly model relations, reuse grammar rules over different scopes, and allow recursive rules. These differences make stochastic AOGs better suited for certain domains and applications, e.g., to model recursion in language and translation invariance in computer vision.

\section{Computational Complexity of Inference}\label{sm:sec:np}
We prove that the parsing problem of stochastic AOGs (i.e., given a data sample consisting of only terminal nodes, finding its most likely parse) is NP-hard.
\begin{theorem}
	The parsing problem of stochastic AOGs is NP-hard.
\end{theorem}
\begin{proof}
	Below we reduce 3SAT to the parsing problem. 
	
	For a 3SAT CNF formula with $n$ variables and $k$ clauses, we construct a stochastic AOG of polynomial size in $n$ and $k$. The node parameters in this AOG always take the value of null (i.e., no parameter), and accordingly in any And-rule of the AOG the parameter relation always returns true and the parameter function always returns null. For each variable $x_i$, create one Or-node $A_i$, two And-node $X_i$ and $\overline{X_i}$, and two Or-rules $A_i \rightarrow X_i | \overline{X_i}$ with equal probabilities. Create an And-rule $S \rightarrow \{A_1, A_2, \ldots, A_n\}$ where $S$ is the start symbol. For each clause $c_j$, create an Or-node $B_j$,  a terminal node $C_j$ and two Or-rules $B_j \rightarrow C_j | \epsilon$ with equal probabilities. Here $\epsilon$ represents the empty set. For each literal $l$ (which can be either $x_i$ or $\overline{x_i}$ for some $i$), suppose $L$ is the corresponding And-node (i.e., $X_i$ or $\overline{X_i}$), if $l$ appears in one or more clauses $c_{h_1}, c_{h_2}, \ldots, c_{h_m}$, then create an And-rule $L \rightarrow \{B_{h_1}, B_{h_2}, \ldots, B_{h_m}\}$; otherwise create an And-rule $L \rightarrow \epsilon$. 
	Note that the constructed AOG does not conform to the standard definition of AOG in that it contains the empty set symbol $\epsilon$ and that some And-rules may have only one child node. However, the constructed AOG can be converted to the standard form with at most polynomial increase in grammar size. See \cite{lange2009cnf} for a list of CFG conversion approaches, which can be extended for AOGs. For simplicity in proof, we will still use the non-standard form of the constructed AOG below.
	
	We then construct a data sample which simply contains all the terminal nodes with no duplication: $\{C_1, C_2, \ldots, C_k\}$.
	
	We first prove that if the 3SAT formula is satisfiable, then the most likely parse of the data sample can be found (i.e., there exists at least one valid parse). Given a truth assignment that satisfies the 3SAT formula, we can construct a valid parse tree. First of all, the parse tree shall contain the start symbol and hence the production $S \rightarrow \{A_1, A_2, \ldots, A_n\}$. For each variable $x_i$, if it is true in the assignment, then the parse tree shall contain production $A_i \rightarrow X_i$; if it is false, then the parse tree shall contain production $A_i \rightarrow \overline{X_i}$. For each clause $c_j$, select one of its literals that are true and suppose $L$ is the corresponding And-node; then the parse tree shall contain productions $L \rightarrow \{\ldots, B_j, \ldots\} \rightarrow \{\ldots, C_j, \ldots\}$, where the first production is based on the And-rule headed by $L$ and the second production is based on Or-rule $B_j \rightarrow C_j$. In this way, all the terminal nodes in the data sample are covered by the parse tree. Finally, for any $B_k$ node (for some $k$) in the parse tree that does not produce $C_k$, add production $B_k \rightarrow \epsilon$ to the parse tree. The parse tree construction is now complete.
	
	Next, we prove that if the most likely parse of the data sample can be found, then the 3SAT formula is satisfiable. For each variable $x_i$, the parse tree must contain either production $A_i \rightarrow X_i$ or production $A_i \rightarrow \overline{X_i}$ but not both. In the former case, we set $x_i$ to true; in the latter case, we set it to false. We can show that this truth assignment satisfies the 3SAT formula. For each clause $c_j$ in the formula, suppose in the parse tree the corresponding terminal node $C_j$ is a descendant of And-node $L$ (which can be $X_i$ or $\overline{X_i}$ for some $i$). Let $l$ be the literal corresponding to And-node $L$. According to the construction of the AOG, clause $c_j$ must contain $l$. Based on our truth assignment specified above, $l$ must be true and hence $c_j$ is true. Therefore, the 3SAT formula is satisfied.
\end{proof}

Another inference problem of stochastic AOGs is to compute the marginal probability of a data sample. The proof above can be easily adapted to show that this problem is NP-hard as well (with the same AOG construction, one can show that the 3SAT formula is satisfiable iff. the marginal probability is nonzero).

\section{Conversion to Generalized Chomsky Normal Form}\label{sm:sec:cnf}
In our inference algorithm, we assume the input AOG is in a generalized version of Chomsky normal form, i.e., (1) each And-node has exactly two child nodes which must be Or-nodes, (2) the child nodes of Or-nodes must not be Or-nodes, and (3) the start symbol $S$ is an Or-node.

By extending previous approaches for context-free grammars \cite{lange2009cnf}, we can convert any AOG into this generalized Chomsky normal form with the following steps. Both the time complexity of the conversion and the size of the new AOG is polynomial in the size of the original AOG.
\begin{enumerate}
	\item (\textbf{START}) If the start symbol is an And-node, create a new Or-node as the start symbol that produces the original start symbol.
	\item (\textbf{BIN}) For any And-rule that contains more than two nodes on the right-hand side, replace the And-rule with a set of binary And-rules, i.e., convert $A \rightarrow \{x_1, x_2, \ldots, x_n\}$ ($n>2$) to $A_1 \rightarrow \{x_1, x_2\}, A_2 \rightarrow \{A_1, x_3\}, \ldots, A \rightarrow \{A_{n-2}, x_n\}$, where $A_i$ are new And-nodes. We will discuss how to convert parameter relation and function later.
	\item (\textbf{UNIT}) For any Or-rule with an Or-node on the right-hand side, $O_1 \rightarrow O_2$, remove the Or-rule and for each Or-rule $O_2 \rightarrow x$ create a new Or-rule $O_1 \rightarrow x$ (unless it already exists in the grammar).
	\item (\textbf{ALT}) If an And-rule contains an And-node or terminal node on the right-hand side, replace the node with a new Or-node that produces the node.
\end{enumerate}

In the \textbf{BIN} step, we have to binarize the parameter relation $t$ and function $f$ along with the production rule, such that:
\begin{align*}
f(\theta_{x_1}, \theta_{x_2}, \ldots, \theta_{x_n}) &= f_A (\theta_{A_{n-2}},\theta_{x_n}) \\
\theta_{A_{n-2}} &= f_{A_{n-2}}(\theta_{A_{n-3}},\theta_{x_{n-1}}) \\
\vdots\\
\theta_{A_2} &= f_{A_2}(\theta_{A_1},\theta_{x_3}) \\
\theta_{A_1} &= f_{A_1}(\theta_{x_1},\theta_{x_2})
\end{align*}
and
\begin{align*}
t(\theta_{x_1}, \theta_{x_2}, \ldots, \theta_{x_n}) \Leftrightarrow t_A (\theta_{A_{n-2}},\theta_{x_n}) 
\land t_{A_{n-2}}(\theta_{A_{n-3}},\theta_{x_{n-1}}) \\
\cdots
\land t_{A_2}(\theta_{A_1},\theta_{x_3}) 
\land t_{A_1}(\theta_{x_1},\theta_{x_2})
\end{align*}
In some cases (e.g., the example AOG of line drawings in the main text), the parameter relation and function can be naturally factorized into this form.
In general, however, we have to cache multiple parameters of the right-hand side nodes of the And-rule in the intermediate parameters $\theta_{A_1}, \theta_{A_2}, \ldots, \theta_{A_{n-2}}$:
\begin{align*}
\theta_{A_1} &= f_{A_1}(\theta_{x_1},\theta_{x_2}) := \langle \theta_{x_1},\theta_{x_2} \rangle \\
\theta_{A_2} &= f_{A_2}(\theta_{A_1},\theta_{x_3}) := \langle \theta_{x_1},\theta_{x_2},\theta_{x_3} \rangle \\
\vdots\\
\theta_{A_{n-2}} &= f_{A_{n-2}}(\theta_{A_{n-3}},\theta_{x_{n-1}}) := \langle \theta_{x_1},\theta_{x_2},\ldots,\theta_{x_{n-1}} \rangle
\end{align*}
then we define 
\[
f_A (\theta_{A_{n-2}},\theta_{x_n}) := f(\theta_{x_1}, \theta_{x_2}, \ldots, \theta_{x_n})
\]
and
\begin{align*}
& t_{A_1}(\theta_{x_1},\theta_{x_2}) = t_{A_2}(\theta_{A_1},\theta_{x_3}) = \cdots = t_{A_{n-2}}(\theta_{A_{n-3}},\theta_{x_{n-1}}) := \top \\
& t_A (\theta_{A_{n-2}},\theta_{x_n}) := t(\theta_{x_1}, \theta_{x_2}, \ldots, \theta_{x_n})
\end{align*}
Note that the sizes of the intermediate parameters can be polynomial in $n$. This actually violates the requirement that the parameter size shall be upper bounded by a constant. Nevertheless, when running our inference algorithm on the resulting Chomsky normal form AOG, the inference time complexity is only slightly affected, with the last factor $(|X|+|G|)$ changed to a function polynomial in $|X|$ and $|G|$, and hence the condition for tractable inference remains unchanged.

\end{document}